\documentclass[letterpaper]{article} 

\usepackage{ifthen}
\newboolean{arxiv}
\setboolean{arxiv}{true}

\usepackage{aaai2026}  
\usepackage{times}  
\usepackage{helvet}  
\usepackage{courier}  
\usepackage[hyphens]{url}  
\usepackage{graphicx} 
\usepackage{natbib}  
\usepackage{caption} 
\usepackage{algorithm}
\usepackage{algorithmic}
\usepackage{subcaption}

\usepackage{newfloat}
\usepackage{listings}
\DeclareCaptionStyle{ruled}{labelfont=normalfont,labelsep=colon,strut=off} 
\floatstyle{ruled}
\newfloat{listing}{tb}{lst}{}
\floatname{listing}{Listing}
\usepackage{enumitem}
\setitemize{itemsep=1pt,topsep=1pt,itemindent=10pt,leftmargin=0pt}
\setenumerate{itemsep=1pt,topsep=1pt,itemindent=10pt,leftmargin=0pt}

\ifthenelse{\boolean{arxiv}}{%
\nocopyright
\title{Compiling Temporal Numeric Planning into Discrete PDDL+:\\Extended Version\footnote{This paper is an extended version of the homonymous appearing in the ICAPS 2026 proceedings. This version provides the proofs and addidional explanations of the compilation.}}
}{%
\title{Compiling Temporal Numeric Planning into Discrete PDDL+}
}

\author {
    Andrea Micheli\textsuperscript{\rm 1},
    Enrico Scala\textsuperscript{\rm 2},
    Alessandro Valentini\textsuperscript{\rm 1}
}
\affiliations {
    \textsuperscript{\rm 1}Fondazione Bruno Kessler, Trento, Italy\\
    \textsuperscript{\rm 2}University of Brescia, Italy\\
    amicheli@fbk.eu, enrico.scala@unibs.it, alvalentini@fbk.eu
}

\usepackage{amsmath}
\usepackage{amssymb}
\usepackage{amsfonts}
\usepackage{amsthm}
\usepackage{xspace}
\usepackage{todonotes}
\usepackage{marvosym}
\usepackage{booktabs}

\newtheorem{definition}{Definition}
\newtheorem{theorem}{Theorem}

\newcommand{\snapstart}[1]{\ensuremath{{#1}_{\vdash}}\xspace}
\newcommand{\snapend}[1]{\ensuremath{{#1}_{\dashv}}\xspace}

\newcommand{\assign}{\mathrel{:}=}
\newcommand{\increment}{\mathrel{+}=}

\newcommand{\eff}[1]{\ensuremath{\mathit{eff}_{#1}}\xspace}
\newcommand{\pre}[1]{\ensuremath{\mathit{pre}_{#1}}\xspace}
\newcommand{\overall}[1]{\ensuremath{\gamma_{#1}}\xspace}
\newcommand{\lb}[1]{\ensuremath{l_{#1}}\xspace}
\newcommand{\ub}[1]{\ensuremath{u_{#1}}\xspace}

\newcommand{\happenings}[1]{\ensuremath{\mathit{H}_{#1}}\xspace}

\newcommand{\vrs}[1]{\ensuremath{\mathit{vars}(#1)}\xspace}
\newcommand{\vrspre}[1]{\ensuremath{V^\mathit{pre}_{#1}}\xspace}
\newcommand{\vrsread}[1]{\ensuremath{V^\mathit{r}_{#1}}\xspace}
\newcommand{\vrswrite}[1]{\ensuremath{V^\mathit{w}_{#1}}\xspace}
\newcommand{\vrsassign}[1]{\ensuremath{V^\mathit{\assign}_{#1}}\xspace}
\newcommand{\vrsincrease}[1]{\ensuremath{V^\mathit{\increment}_{#1}}\xspace}

\newcommand{\defas}{\ensuremath{\stackrel{\text{\tiny def}}{=}}\xspace}

\newcommand{\der}[1]{\ensuremath{\frac{d}{dt}{#1}}\xspace}
\newcommand{\affects}{\ensuremath{\gets}}

\newcommand{\alive}{\ensuremath{\mathit{ok}}\xspace}
\newcommand{\opencounter}{\ensuremath{\mathit{oc}}\xspace}
\newcommand{\gc}{\ensuremath{\mathit{gc}}\xspace}
\newcommand{\running}[1]{\ensuremath{r_{#1}}\xspace}
\newcommand{\clock}[1]{\ensuremath{c_{#1}}\xspace}

\newcommand{\rlock}[1]{\ensuremath{{rl}_{#1}}\xspace}
\newcommand{\alock}[1]{\ensuremath{{al}_{#1}}\xspace}
\newcommand{\ilock}[1]{\ensuremath{{il}_{#1}}\xspace}

\newcommand{\lockformula}[1]{\ensuremath{{\lambda_{#1}}}\xspace}
\newcommand{\lockeffects}[1]{\ensuremath{{L_{#1}}}\xspace}

\newcommand{\mutex}{\text{\Lightning}}

\begin{document}

\maketitle

\begin{abstract}
    Since the introduction of the PDDL+ modeling language, it was known that temporal planning with durative actions (as in PDDL 2.1) could be compiled into PDDL+. However, no practical compilation was presented in the literature ever since. We present a practical compilation from temporal planning with durative actions into PDDL+, fully capturing the semantics and only assuming the non-self-overlapping of actions. Our compilation is polynomial, retains the plan length up to a constant factor and is experimentally shown to be of practical relevance for hard temporal numeric problems. 
\end{abstract}

\section{Introduction}

Automated planning is the task of finding a course of actions to achieve a goal from an initial state given a model of the system specifying which actions are available, together with (i) their precondition, what needs to hold in order for an action to be applicable and (ii) their effect, what needs to hold when such actions are applied. Temporal planning is the extension in which actions are assumed to last for some interval of time, and so we look for a plan that is schedulable too, i.e., actions need to be done in specific points over a potentially unbounded timeline, with conditions required to hold at the beginning, at the end and during the execution of the action. Temporal planning problems can be compactly represented in the PDDL 2.1 language \cite{pddl21}. 

A number of solutions have been proposed to handle PDDL2.1 problems, ranging from forward heuristic search in the space of possible schedules \cite{popf,optic,tamer} to satisfiability-based bounded reductions \cite{patty}. Among these approaches, compilation-based techniques are particularly appealing, as they allow temporal reasoning to be delegated to more expressive or better-supported target languages. Building on this line of work, this paper studies a compilation method that transforms a temporal planning problem specified in PDDL2.1 into an equivalent formulation in PDDL+.


PDDL+ is yet another extension of classical planning which provides a different take to the problem of representing timed and hybrid systems. Instead of having durative actions, in PDDL+ a system over time is modeled through a combination of processes, events and instantaneous actions. Processes model the system evolving over time through differential equations, and events dictate what needs to change instantaneously if some condition is satisfied. 

Our compilation is rooted in the known observation that a durative action can be compiled into a combination of processes and events/actions \cite{pddlplus}, but we contribute the first fully spelled-out formal account to approach this problem rigorously. Moreover, the expressiveness of PDDL+ makes it easier to extend the input problems with a number of features that are not well supported by many temporal planners, such as numeric state variables, delayed effects and timed initial literals. Our compilation provides a comprehensive account for the full semantics of temporal planning resulting in an encoding that is sound, complete and polynomial on the size of the input problem.

To shed some light on the practical benefit of this compilation we run an experimental campaign on a number of temporal numeric domains. Surprisingly, PDDL+ planners prove competitive, and often superior, to state-of-the-art temporal planners on rich numeric temporal problems. This may inform a more precise understanding of the core difficulties in temporal planning.

\section{Background}

We start by defining a temporal planning problem, adapting the PDDL 2.1 level 3 language \cite{pddl21}\ifthenelse{\boolean{arxiv}}{\footnote{
We speculate that our compilation could be easily adapted to PDDL 2.1 level 4, hence including continuous change, but for the sake of simplicity we restrict the paper presentation to level 3.
}}{}.

\begin{definition}
    A \textbf{temporal planning problem} $\mathcal{P}^t$ is a tuple $(F, X, I, A^i, A^{d}, G)$ where:
    \begin{itemize}
        \item $F$ is a finite set of boolean fluents (predicates);
        \item $X$ is a finite set of numeric rational fluents;
        \item $I: F \cup X \rightarrow \mathbb{B} \cup \mathbb{Q}$ is the initial state, assigning each fluent to its initial value;
        
        \item $A^i$ is a finite set of instantaneous actions; each  $a \in A^i$ has a precondition formula \pre{a} over $F \cup X$ and a set of boolean and numeric effects \eff{a} of the form $f \assign \{\bot, \top\}$ if $f \in F$ or $f \affects e$ with $\affects \in \{\assign, \increment\}$ and $e$ being a numeric expression over $X$, if $f \in X$. 
        
        \item $A^{d}$ is a finite set of durative actions; each $a \in A^{d}$ has lower and upper duration bounds $\lb{a} \le \ub{a} \in \mathbb{Q}_{> 0}$, a pair of starting \snapstart{a} and ending \snapend{a} instantaneous ``snap'' actions, and an overall invariant formula $\overall{a}$ over $F \cup X$.
        \item $G$ is the goal expressed as a formula over $F \cup X$.
    \end{itemize}
\end{definition}

\noindent
We partition $A^d$ into $A^{fix} \sqcup A^{var}$, with $A^{fix} = \{a \in A^d \mid \lb{a} = \ub{a}\}$ being the set of durative actions with a fixed duration (therefore, $A^{var}$ is the set of durative actions with variable duration). Moreover, given a (snap) instantaneous action $a$, we write: $\vrspre{a} \defas \vrs{\pre{a}}$ for the set of fluents occurring in $\pre{a}$; $\vrsread{a} \defas \vrspre{a} \cup \bigcup_{(f \affects e) \in \eff{a}} \vrs{e}$ for the fluents read by the precondition or by any effect of $a$; $\vrsassign{a} \defas \bigcup_{(f \assign e) \in \eff{a}} \{f\}$ for the assigned fluents; $\vrsincrease{a} \defas \bigcup_{(f \increment e) \in \eff{a}} \{f\}$ for the increased fluents; and $\vrswrite{a} \defas \vrsassign{a} \cup \vrsincrease{a}$ for the effected (written) fluents.

\begin{definition}
    A temporal plan $\pi^{t}$%
    is a finite set of triples of the form $(t, a, d)$, where $t \in \mathbb{Q}_{\ge 0}$ is the starting time, $a \in A^i \cup A^d$ is the action to execute and $d \in \mathbb{Q}_{\ge 0}$ is the action duration, s.t. $\lb{a} \le d \le \ub{a}$ if $a \in A^d$, and $d = 0$ if $a \in A^i$.
\end{definition}

We present an adaptation of the non-self-overlapping semantics by \citet{gigante22}. A state is a total assignment of values to fluents; given a formula $e$ defined over $F \cup X$, we write $s(e)$ for the value of $e$ in $s$ obtained by substitution and constant propagation. Given a state $s$ and an instantaneous action $a$, $a$ is applicable in $s$ if $s(\pre{a})$ is true ($s \models \pre{a}$) and the successor state $s' \defas a(s)$ 
%
is such that $s'(f) = s(e)$ if $f \assign e \in \eff{a}$, $s'(f) = s(f) + s(e)$ if $f \increment e \in \eff{a}$, $s'(f) = s(f)$ otherwise.

Let $\pi^{t} \defas \{(t_1, a_1, d_1), \ldots, (t_n, a_n, d_n)\}$ be a temporal plan for a planning problem $\mathcal{P}^t = (F, X, I, A^i, A^{d}, G)$. 
Let $H^{\pi^t}$ be the set of timed (snap) actions for $\pi^{t}$ defined as $\{(t_i, a_i) \mid a_i \in A^i\} \cup \{(t_i, a_{i,\vdash}) \mid a_i \in A^d\} \cup \{(t_i + d_i, a_{i,\dashv}) \mid a_i \in A^d\}$. Let $t^s_0, t^s_1, \ldots, t^s_{m-1} \in \mathbb{Q}$ be the times appearing in  $H^{\pi^t}$ (that is, $\exists (t, a) \in H^{\pi^t} . t^s_j = t$) ordered s.t. $t^s_j < t^s_{j+1}$; moreover, we add an arbitrary final time: $t^s_m \defas t^s_{m-1} + 1$. We define the set of ``happenings'' at step $j$ as the set $\happenings{j} = \{a \mid t^s_j = t \mbox{ and } (a, t) \in H^{\pi^t}\}$.
We can now give the semantics of temporal planning: the plan $\pi^{t}$ is valid if there exists a sequence of states $s_0, \ldots, s_m$ such that:
\begin{enumerate}
    \item $s_0 = I$ and $s_{m} \models G$ (initial state and goal constraints),
    \item for each $0 \le j < m$ and each $a \in \happenings{j}$, $s_j \models \pre{a}$ (all conditions for happenings at step $j$ are satisfied);
    \item for each $0 \le j < m$, $s_{j+1} = b_0(b_1(\cdots b_k(s_j)))$, with $\happenings{j} = \{b_0, b_1, \ldots, b_k\}$ for an arbitrary ordering (a state is the result of applying all the effects of all happenings);
    \item for each $1 \le i \le n$ with $a_i \in A^d$, if $t^s_j = t_i$ and $t^s_k = t_i + d_i$, $s_w \models \gamma(a_i)$ for all $j < w \le k$ (overall conditions);
    \item for each $0 \le j < m$ and each pair of instantaneous actions $a \not= b \in \happenings{j}$, $\vrsread{a} \cap \vrswrite{b} = \vrsread{b} \cap \vrswrite{a} = \vrsassign{a} \cap \vrsassign{b} = \emptyset$ (no interfering actions at the same time);
    \item For all $j \not= k \in [1, n]$, $t_j > t_k+d_k$ or $t_k > t_j+d_j$ (no-self-overlapping constraint).
\end{enumerate}

The target of our compilation is a PDDL+ \cite{pddlplus} problem, formalized below.
\begin{definition}
    A \textbf{PDDL+ problem} is modeled as a tuple $(F, X, I, G, A, E, P)$ where:
    \begin{itemize}
        \item $F$ is a finite set of boolean fluents (predicates);
        \item $X$ is a finite set of numeric rational fluents;
        \item $I: F \cup X \rightarrow \mathbb{B} \cup \mathbb{Q}$ is the initial state, assigning each fluent to its initial value;
        \item $G$ is a goal condition expressed as a formula over $F \cup X$.        
        \item $A$, $E$ are two finite sets of instantaneous actions and events resp.; each action/event $x$ is defined by a set of preconditions \pre{x} and a set of effects \eff{x}.
        \item $P$ is a set of processes each $p \in P$ having a precondition formula \pre{p} over $F \cup X$ and a set of effects \eff{p} of the form $\der{f} \increment e$ with $f \in X$ and $e$ is a formula over $X$. 
    \end{itemize}
\end{definition}

\begin{definition}
    A \textbf{PDDL+ plan} is a pair $(\pi^+, t_e)$, where $t_e \in \mathbb{Q}$ is the plan makespan and $\pi^+$ is a finite sequence of pairs $(t_i, a_i)$ with $t_i \le t_e \in \mathbb{Q}_{\ge0}$.    
\end{definition}

Following \citet{pddlplus-discrete}, we assume a discretization time quantum $\delta \in \mathbb{Q}$. A plan is well-formed if $t_e$ and all times $t_i$ are multiples of $\delta$. A plan $(\pi^+, t_e)$ is valid for a PDDL+ problem $(F, X, I, G, A, E, P)$ if there exists a sequence of states (similarly to the temporal case) $\bar{s}_0, \ldots, \bar{s}_{\bar{m}}$ with $\bar{s}_0 = I$ and $\bar{m} = \frac{t_e}{\delta}$ defined as follows.
Let $\bar{t}^s_j = \delta \cdot j$ be the time of state $\bar{s}_j$ and let the sequence of ``action happenings'' $\happenings{j}^a$ at step $j$ be the sequence of actions from $\pi^+$ happening at time $\bar{t}^s_j$;
i.e., $\happenings{j}^a = (a_i,\dots,a_{i+k})$ where $(t_i,a_i),\ldots,(t_{i+k},a_{i+k})$ is a maximal sub-sequence of $\pi^+$ with $t_i = \cdots = t_{i+k} = \bar{t}^s_j$.
Given a state $\bar{s}$, we inherit the definitions of action applicability from above and generalize them for events in the obvious way. We define the event completion of $\bar{s}$ (written $\bar{s}^\rightarrow$) as the fixed-point state reached after applying any applicable event in $\bar{s}$ and then any other applicable event in the resulting state, until no event is applicable anymore.
For each $j$, we define the final state at step $j$ as $\bar{s}^{end}_j \defas b_0(b_1(\cdots(b_k(\bar{s}_j^\rightarrow)^\rightarrow)^\rightarrow)^\rightarrow$ with $\happenings{j}^a = (b_0, b_1, \ldots, b_k)$ (intuitively, we apply all  actions ordered by $\pi^+$ to the event completion of $\bar{s}_j$ and after every action we perform an event completion). We can now define the state transition: for every $0 \le j < \bar{m}$, we define $\bar{s}_{j+1}(f) = \bar{s}^{end}_j(f)$ for every $f \in F$, and for every $x \in X$: 
$$
\bar{s}_{j+1}(x) = \bar{s}^{end}_j(x) + \sum_{\substack{p \in P, \: (x \increment e) \in \eff{p} \\ \bar{s}^{end}_j \models \pre{p}}} \bar{s}^{end}_j(e) \cdot \delta
$$
Intuitively, we set every predicate to its final value at step $j$ and compute the value of the numeric fluents after some passage of time by applying a discretized step aggregating the contribution of every active process. 
Finally, the plan is valid if $\bar{s}_m \models G$, and each $a_i \in \happenings{j}^a$ is applicable in $\bar{s}_j^\rightarrow$.

\section{Compile Temporal Planning into PDDL+}


In this section, we formally define our compilation from temporal planning into PDDL+. 
Intuitively, we formulate a PDDL+ problem where each durative action is emulated by a triplet (action, process, event) for fixed duration actions, and (action, process, action) for flexible duration ones. To only encode valid temporal plans, we need all temporal points from 1-6 (see previous section) to hold. We introduce a number of auxiliary boolean and numeric state variables, and use them to constrain actions properly, and propagate processes and events when necessary. A key step is the introduction of lock preconditions-effects, a machinery employed to prevent interfering actions to happen at the same time.
Below, we formalize the encoding.

In the following, we assume a temporal planning problem $\Pi \defas (F, X, I, A^i, A^{d}, G)$ is given, and we define the compiled PDDL+ problem $\bar{\Pi} \defas (\bar{F}, \bar{X}, \bar{I}, \bar{G}, \bar{A}, \bar{E}, \bar{P})$.

\noindent
Let $FX \defas F \cup X$; we start with the fluents definition.
\begin{align*}
    \bar{F} & \defas F \cup \{\alive\} \cup \{\running{a} \mid a \in A^{d}\}\cup \{\rlock{f}, \alock{f}, \ilock{f} \mid f \in FX\} \\
    \bar{X} & \defas X \cup \{\opencounter, \gc\} \cup \{\clock{a} \mid a \in A^{d}\}
\end{align*}
\noindent
In addition to the fluents of $\Pi$, we add a new \alive predicate that will be required by every action in the compiled model and by the new goal, this will be used to ``abort'' a plan that violated some constraints. We also add two numeric fluents, \opencounter and \gc, to count the number of actions started but not terminated and to distinguish the time in consecutive steps, respectively. Moreover, for every durative action $a$, we add a predicate \running{a}, which will be kept to true while a durative action is running, and a numeric fluent \clock{a} which will be used to measure the time since the start of $a$. Finally, for every fluent $f$ we define three
%
%
``lock'' predicates, \rlock{f}, \alock{f} and \ilock{f}, that will be used to enforce mutual exclusion constraints.

The initial state and goal condition are defined by simply augmenting the original initial state and goal as follows.
\begin{align*}
    \bar{I} & \defas \{\alive \!=\! \top, \opencounter\!=\!\gc\!=\!0\} \cup \{\running{a} \!=\! \bot, \clock{a} \!=\! 0 \mid a \in A^d \} \: \cup\\
                   & \quad \:\:\{\rlock{f} \!=\! \alock{f} \!=\! \ilock{f} \!=\! \top \mid f \in FX\} \cup I\\
    \bar{G} & \defas G \wedge \alive \wedge \opencounter = 0
\end{align*}

Before defining the rest of the compilation, the following definition provides the key machinery for encoding the non-interference constraints.

\begin{definition}
    Given an instantaneous action $a$, we define the \textbf{lock precondition} $\lockformula{a}$ as:
    \begin{align*}
    \bigwedge_{f \in \vrsread{a}} \! (\alock{f} \wedge \ilock{f}) \wedge \!\!\! \bigwedge_{f \in \vrsassign{a}} \!\! (\rlock{f} \wedge \ilock{f} \wedge \alock{f}) \wedge \!\!\! \bigwedge_{f \in \vrsincrease{a}} \!\!\! (\rlock{f} \wedge \alock{f}) .
    \end{align*}

    \noindent
    Moreover, we define the \textbf{lock effects} $\lockeffects{a}$ as the set:
    \begin{align*}
    & \{\alock{f} \assign \bot \mid f \in \vrsassign{a}\} \cup \{\ilock{f} \assign \bot \mid f \in \vrsincrease{a}\} \cup \\
    & \{\rlock{f} \assign \bot \mid f \in \vrsread{a}\}.
    \end{align*}
\end{definition}

\noindent
Intuitively, we will augment every happening associated with an instantaneous action or with the start or end of a durative action with its lock precondition and effects; if all locks are reset to $\top$ at every time step (see $\bar{e}^\mutex$ below), no pair of interfering happenings can appear at the same time.

We now define the core of the translation starting from actions: we introduce a PDDL+ action for every instantaneous action, start of durative action and termination of non-fixed durative action in the original problem: $\bar{A} = \{\bar{a}^i \mid a \in A^i\} \cup \{\bar{a}^{\vdash} \mid a \in A^d\} \cup \{ \bar{a}^\dashv \mid a \in A^{var}\}$ defined as follows.
\begin{align*}
    \bar{a}^i \defas (\pre{a} \wedge \alive \wedge \lockformula{a}, \eff{a} \cup \lockeffects{a})
\end{align*}
Instantaneous actions are simply augmented with the \alive precondition (common to all other actions and events in the compilation) and with the lock preconditions and effects. 

Durative actions are split into their starting and ending timepoints; fixed-duration actions will be terminated by an event ($\bar{e}^\dashv_a$ defined below), while variable-duration actions are terminated by the $\bar{a}^\dashv$ actions.
\begin{align*}
    \bar{a}^{\vdash} \defas ( & \pre{\snapstart{a}} \wedge \alive \wedge \lockformula{\snapstart{a}} \wedge \neg \running{a},\\
        & \eff{\snapstart{a}} \cup \lockeffects{\snapstart{a}} \cup \{\running{a} \assign \top, \clock{a} \assign 0, \opencounter \increment 1\})
\end{align*}
The starting of every durative action $a$ corresponds to the $\snapstart{a}$ snap action, we add the $\neg \running{a}$ precondition to enforce non-self-overlapping and we set $\running{a}$ to true, we reset the clock \clock{a} because we are just starting the action, and we increase \opencounter to signal that a new action started but has not finished yet.
\begin{align*}
\bar{a}^\dashv \defas ( & \pre{\snapend{a}} \wedge \alive \wedge \running{a} \wedge \lb{a} \le \clock{a} \le \ub{a} \wedge \lockformula{\snapend{a}},\\
        & \eff{\snapend{a}} \cup \lockeffects{\snapend{a}} \cup \{\running{a} \assign \bot, \opencounter \increment -1 \})
\end{align*}
To terminate a non-fixed durative action $a$, we require \running{a} and that the duration constraint is satisfied with $\lb{a} \le \clock{a} \le \ub{a}$; we then reset \running{a} to false and decrement \opencounter.

We have two types of very simple processes in our compilation that are used to keep track of the time passing while an action is running and to continuously increase \gc for the mutex construction we will describe below. The compilation processes are then $\bar{P} = \{\bar{p}_a \mid a \in A^d\} \cup \{\bar{p}^{\mutex}\}$ with:
\begin{align*}
    \bar{p}_a \defas (\alive \wedge \running{a}, \{\frac{d}{dt} \clock{a} = 1\}) \quad\quad
    \bar{p}^{\mutex} \defas (\alive, \{\frac{d}{dt} \gc = 1 \})
\end{align*}

In the compilation, events serve several purposes, each encoded in a different subset: $\bar{E} \defas \bar{E}^\leftrightarrow \cup \bar{E}^{fix}_\dashv \cup \bar{E}^{expire} \cup \{\bar{e}^{\mutex}\}$.
$\bar{E}^{fix}_\dashv \defas \{\bar{e}_a^\dashv \mid a \in A^{fix}\}$ encodes the termination of fixed-duration actions, $\bar{E}^\leftrightarrow \defas \{\bar{e}_a^\leftrightarrow \mid a \in A^{d}\}$ ensures that overall conditions of durative actions are not violated, $\bar{E}^{expire}$ ensure that no plan prefix has variable duration actions that are not terminated within the duration upper bound and  $\bar{e}^{\mutex}$ resets all the lock variables immediately after a time-elapse.
\begin{align*}
\bar{e}_a^\dashv \defas ( & \pre{\snapend{a}} \wedge \alive \wedge \running{a} \wedge \clock{a} = \lb{a} \wedge \lockformula{\snapend{a}} \wedge \gc = 0,\\
        & \eff{\snapend{a}} \cup \lockeffects{\snapend{a}} \cup \{\running{a} \assign \bot, \opencounter \increment -1 \})
\end{align*}
The termination of fixed-duration actions is analogous to the variable duration ones, but is an event scheduled at the fixed duration ($\lb{a} = \ub{a}$), measured by \clock{a}. We also require $\gc$ to be $0$ to execute this event after $\bar{e}^{\mutex}$, as explained below.
\begin{align*}
    \bar{e}_a^\leftrightarrow \defas (\alive \wedge \running{a} \wedge \neg \overall{a}, \{\alive \assign \bot \})
\end{align*}
Overall conditions are enforced through the \alive predicate: if we reach a state where action $a$ is running ($\running{a}$ is true) and its overall conditions are not satisfied, we set \alive to false, making this prefix invalid, because \alive can only be falsified, and never restored to true.
Similarly, if a durative action with variable duration is not terminated within its duration upper bound ($\running{a} \wedge \clock{a} > u$), we immediately set \alive to false. (This is not strictly needed for correctness, as \gc would never return to 0, but is useful for performance.)
\begin{align*}
    \bar{E}^{expire} \defas \{ & (\alive \wedge \running{a} \wedge \clock{a} > \ub{a}, \{\alive \assign \bot \} ) \mid a \in A^{var}\}    
\end{align*}
Finally, we have a single event that restores the locks to true whenever \gc is positive as follows. 
\begin{align*}
    \bar{e}^{\mutex} \defas ( & \alive \wedge \gc > 0,  \{\gc \assign 0\} \cup \\
    & \{\rlock{x} \assign \top, x_{wl} \assign \top, \ilock{x} \assign \top \mid x \in FX\})
\end{align*}
The idea of this ``lock'' construction is that in a certain time we start the chains of happenings with the event $\bar{e}^{\mutex}$ that resets all the locks, then we can execute other actions or events, but every time we ``read'' a fluent $f$ (either in a precondition or in the right-hand-side of an effect) we set $\rlock{f}$ to false, every time we have an assignment or increment effect of $f$ we set \alock{f} or \ilock{f} respectively to false. Thanks to the lock preconditions, we forbid reading a variable if it was previously (in the same ``superdense'' time) assigned or incremented, we forbid an assignment if it was previously assigned, increased or read, and we forbid increments if it was previously assigned or read. The whole trick is that \gc will be continuously increased by a process, so in the subsequent times, the locks are automatically reset by the event  $\bar{e}^{\mutex}$ and the locks are released. This faithfully captures the semantics of non-interference we outlined for temporal planning.

Given a plan $(\bar{\pi}^+, t_e)$ with $\bar{\pi}^+ = {(t_1, a_1), \ldots (t_n, a_n)}$ for $\bar{\Pi}$, we define the temporal plan $\tilde{\pi}$ solving $\Pi$ as:
\begin{align*}
    \tilde{\pi} \defas & \{(t, a, 0) \mid (t, \bar{a}^i) \in \bar{\pi}\} \cup \\
    & \{(\snapstart{t}, a, \lb{a}) \mid (\snapstart{t}, a^\vdash) \in \bar{\pi} \mbox{ and } a \in A^{fix}\} \cup \\
    & \{(\snapstart{t}, a, \snapend{t} - \snapstart{t}) \mid (\snapstart{t}, a^\vdash), (\snapend{t}, a^\dashv) \in \bar{\pi} \mbox{ and }\\
    & \quad \not \exists t . \snapstart{t} < t' < \snapend{t}, (t', a^\vdash) \in \bar{\pi}\}
\end{align*}
%

\ifthenelse{\boolean{arxiv}}{
\paragraph{Additional intuition on the lock mechanism.}
The $\gc$ function is meant to keep track of a non zero passage of time, necessary to separate mutex actions. We highlight that $\gc > 0$ is true \emph{at the beginning of the sequence of happenings at any time $t > 0$}. We recall that PDDL+ has a superdense model of time, so a fluent can change its value by means of contemporary happenings, without the time passing. In our construction, at each time $t>0$ we start our sequence of contemporary happenings in a state where $\gc > 0$, because the process $\bar{p}^{\mutex}$ increases it with a positive derivative (we chose 1 for simplicity, but any positive derivative would do); however, immediately at the second superdense step (at the same time) the event $\bar{e}^{\mutex}$ triggers, resetting $\gc$ to 0.
Visualizing the behavior of $gc$ in the superdense time, we have the following:
\begin{itemize}
    \item $\gc(0, 0) = 0$, $\gc(0, 1) = 0$, $\ldots$, $\gc(0, n_0) = 0$;
    \item $\gc(\delta, 0) = \delta$, $\gc(\delta, 1) = 0$, $\ldots$, $\gc(\delta, n_1) = 0$;
    \item $\gc(2\delta, 0) = \delta$, $\gc(2\delta, 1) = 0$, $\ldots$, $\gc(2\delta, n_2) = 0$;
    \item $\cdots$;
    \item $\gc(t, 0) = \delta$, $\gc(t 1) = 0$, $\ldots$, $\gc(t, n_t) = 0$.
\end{itemize}
Where $\gc(t, i)$ indicates the value of $\gc$ at time $t$ and superdense step $i$.

Importantly, this construction motivates our choice of adopting a discrete time semantics for PDDL+. In fact, this construction works perfectly in discrete time (and is not very costly for the planner), but exhibits a Zeno behavior in continuous time, because the event $\bar{e}^{\mutex}$ would trigger \emph{immediately after} the previous happening imposing an infinite number of happenings in a finite amount of time. Finally, note that with an $\epsilon$-separation semantics \cite{gigante22}, the system is essentially discrete, hence the construction will work.
}{}

\noindent
\ifthenelse{\boolean{arxiv}}{Below}{In the extended version of this paper~\cite{extended}}, we prove that the compilation is sound and complete. Indeed, if a plan $(\bar{\pi}^+, t_e)$ is found for $\bar{\Pi}$, so is $\tilde{\pi}$ for $\Pi$. This follows by proving that constraints 1-6 for the validity of a plan are all implied by the existence of a state sequence induced by $(\bar{\pi}^+, t_e)$.  Completeness is more challenging: we prove that for every temporal plan, there exists a sufficiently small $\delta$ for which there is a corresponding valid PDDL+ plan for $\bar{\Pi}$. Finally, we highlight that the compilation is polynomial in size and the plan length is at most doubled.

\ifthenelse{\boolean{arxiv}}{
\section{Theoretical Properties}

In this section, we formally prove the soundness and completeness of our compilation.

\begin{theorem}[Soundness]
    Let $(\bar{\pi^+}, t_e)$ be a valid plan for $\bar{\Pi}$, then $\tilde{\pi}$ is a valid plan for $\Pi$.
\end{theorem}
\begin{proof} 
Since $(\bar{\pi^+}, t_e)$ is a valid plan for a PDDL+ problem, let $\bar{s}_0, \ldots \bar{s}_{\bar{m}}$ bet the states induced by the plan by the discrete semantics of PDDL+. 

We have to show that the six semantic conditions for $\tilde{\pi}$ outlined in the semantics explanation of temporal planning hold. 

First, we ensure the syntactical properties of $\tilde{\pi}$. It is easy to see that every instantaneous action in the plan has a duration of 0 by definition; moreover, the duration $d_a$ of every durative action with fixed duration $a \in A^{fix}$ also trivially satisfies $\lb{a} \le d_a \le \ub{a}$. For durative actions with non-fixed duration, we set the duration as the temporal distance between a consecutive pair of $(\bar{a}^\vdash, \bar{a}^\vdash)$ in $\bar{\pi^+}$. This is guaranteed to be in the $[\lb{a}, \ub{a}]$ interval because:
\begin{itemize}
    \item $\bar{a}^\vdash$ requires $\neg \running{a}$ and resets the counter $\clock{a}$ to $0$ and $\running{a}$ to true, while $\bar{a}^\vdash$ requires $\running{a}$ and $\lb{a} \le \clock{a} \le \ub{a}$
    \item The process $p_a$ starts when $\running{a}$ is set to true and increases $\clock{a}$ with derivative 1
\end{itemize}
Hence, the temporal difference between $\bar{a}^\vdash$ and $\bar{a}^\vdash$ is equal to the value of $\clock{a}$ in $\bar{a}^\vdash$ which is $\lb{a} \le \clock{a} \le \ub{a}$ by precondition.

To prove the semantic conditions, let $\bar{j} \defas \frac{t^s_j}{\delta}$ and we define a sequence of states $s_0, \ldots, s_m$ defined as:
$$
s_j = \bar{s}_{\bar{j}} \mid F \cup X
$$
that is, the j-th temporal state is the PDDL+ state at step $\frac{t^s_j}{\delta}$ restricted to the original problem fluents.
Note that this PDDL+ state exists because $t^s_j$ is derived from times (or differences of times) in $\bar{\pi}^+$ which are integer multiples of $\delta$ by definition.

Another preliminary consideration is that $\happenings{j} = \{\snapstart{a}, \snapend{b} \mid \bar{a}^\vdash, \bar{b}^\dashv \in \happenings{\bar{j}}^a\} \cup \{\snapend{a} \mid (t, a, \lb{a}) \in \bar{\pi}^+, a \in A^{fix}, t^s_j = t + \lb{a}\}$. This is by definition of $\tilde{\pi}$.

We will now show that semantic conditions 1-6 hold for $s_0, \ldots, s_m$.
\begin{enumerate}
    \item $s_0 = I$ because also $I \subseteq \bar{s}_0$. Moreover, $s_0 \models G$ because $\frac{t^s_m}{\delta} = \bar{m}$ and $\bar{s}_{\bar{m}} \models \bar{G}$ and $\bar{G} \rightarrow G$.
    
    \item For every $j$, the preconditions $\pre{a}$ of all $a \in \happenings{j}$ are trivially satisfied for the instantaneous actions, for all the starting snap actions and for $\snapend{a}$ with $a \in A^{var}$, because \pre{a} is also part of the preconditions of the corresponding activities in $\bar{A}$ which are also in $\happenings{\bar{j}}^a$ as noted above and thus true in $\bar{s}_{\bar{j}}$. The only case needing care is the termination of actions with fixed duration, which are encoded by means of the events $\bar{e}_a^\dashv$. It suffices to prove that for every $\bar{a}^\vdash \in \happenings{\bar{k}}^a$ with $a \in A^{fix}$, there is an instance of event $\bar{e}_a^\dashv$ guaranteed to be executable in $s_{\bar{k} + \frac{\lb{a}}{\delta}}^a$, because this event also checks the preconditions of $a$. Suppose that one or more of these event instances is not applied, then the plan $(\bar{\pi^+}, t_e)$ cannot be valid for the PDDL+ semantics, because of the $\opencounter$ dynamics, which counts the difference in number of opening actions and closing actions or events. The only way to apply a closing action is paired with an opening, so if we skip a single instance of $\bar{e}_a^\dashv$, we will result in a final non-zero $\opencounter$ value. Moreover, such event must happen at step $\bar{k} + \frac{\lb{a}}{\delta}$ because of the preconditions on $\clock{a}$ (identical reasoning as per the variable action durations).    
    
    \item Comparing the two semantics, it is easy to see that for every $j$, $\bar{s}_{\bar{j}}^{end}$ is computed analogously to $s_{j+1}$, noting again that ending events corresponding to ending snap actions in $\happenings{j}$ must be executed in $\bar{s}_{\bar{j}}^{end}$. Hence, $s_{j+1} = \bar{s}_{\bar{j+1}} \mid F \cup X$, because fluents in $F \cup X$ are unaffected by processes.
    
    \item Suppose, for the sake of contradiction, that there exists a state $s_w \not\models \overall{a}$ with $j < w \le k$ and $a$ starting in $t^s_j$ and ending in $t^s_k$. Note that every PDDL+ state $\bar{s}_i$ with $\bar{j} < i \le \bar{k}$ is such that $\bar{s}_i(\running{a})$ is true because of the effects of the compiled snap actions as above. Hence, \overall{a} must be true in  $s_w$, otherwise $e_a^\leftrightarrow$ would trigger, making the plan invalid because of the $\alive \assign \bot$ effect that would make the goal unreachable.
    
    \item Suppose, for the sake of contradiction, that there is a pair of actions $a \not= b$ in $\happenings{j}$ (for some $j$) violating the mutex constraint. Then, one of the following three cases must occur.
    \begin{itemize}
        \item $\exists f \in \vrsread{a} \cap \vrswrite{b}$ This is prevented by the preconditions and effects of the compiled actions or events corresponding to $a$ and $b$ (that we indicate with $\bar{a}$ and $\bar{b}$). By corresponding we mean that if $a$ is an instantaneous action, $\bar{a} = \bar{a}^i$; if $a$ is a starting snap action, $\bar{a} = \bar{a}^\vdash$, if $a$ is a ending snap action, either $\bar{a} = \bar{a}^\dashv$ or $\bar{a} = \bar{e}_a^\dashv$, and analogously for $b$.
        
        Now suppose $\bar{a}$ is before $\bar{b}$ in $\bar{\pi}^+$, $\rlock{f}$ is required to be true by $\lockformula{b}$, but $\rlock{f}$ is set to false by $\bar{a}$, and cannot be set to true if not by letting time advance. If the order in $\bar{\pi}^+$ is reversed, both \ilock{f} and \alock{f} are required to be true by $\bar{a}$, but either of them is set to false by the effects of $\bar{b}$ (because $f$ is either assigned or incremented). 
        
        Note that once \alock{f}, \ilock{f} or \rlock{f} is set to false, only $\bar{e}^{\mutex}$ can reset them, but this cannot be done after either $\bar{a}$ or $\bar{b}$ are executed, because $\bar{e}^{\mutex}$ requires $\gc$ to be positive, while both $\bar{a}$ and $\bar{b}$ require it to be 0 and only the process $\bar{p}^\mutex$ can set $\gc$ to a positive value, but only in the subsequent step. (Practically, $\bar{e}^{\mutex}$ is the first event of every step, because $\bar{p}^\mutex$ resets to true its precondition when transitioning to anew step by letting time elapse). This leads to the contradiction. 
        \item $\exists f \in \vrsread{b} \cap \vrswrite{a}$ This case is the symmetric of the previous one.
        \item $\exists f \in \vrsassign{a} \cap \vrsassign{b}$ Similarly to the previous case, the first happening sets \alock{f} to false and both require it to be true, leading to the contradiction for the same reasoning as above.
    \end{itemize}
    
    \item Non self-overlapping is an immediate consequence of the $\running{a}$ dynamics: we forbid to start an action if $\running{a}$ is true and $\running{a}$ is kept true exactly during each action instance.
\end{enumerate}

\end{proof}

\begin{theorem}[Completeness]
    Let $\Pi$ be a temporal planning problem admitting a valid plan $\pi^t$, then there exists a $\delta$ under which there is a valid plan $\bar{\pi}$ for $\bar{\Pi}$.
\end{theorem}
\begin{proof}(Sketch)
    Let $\pi^t = \{(t_1, a_1), \cdots, (t_n, a_n)\}$ and let its induced sequence of states $s_0, \ldots, s_m$ with times $t^s_0, \ldots t^s_m$.
    We define the PDDL+ plan $(\bar{\pi}^+, t_e)$ as $t_e \defas t^s_m + 1$ and 
    $$
        \bar{\pi}^+ \defas ( (t^s_j, \bar{h}) \mid h \in \happenings{j}, 0 \le j < m)
    $$
    with
    $$
    \bar{h} \defas
    \begin{cases}
        \bar{a}^i & \mbox{if } h \in A^i \\
        \bar{a}^\vdash & \mbox{if } h = \snapstart{a} \\
        \bar{a}^\dashv & \mbox{if } h = \snapend{a} \mbox{ and } a \in A^{var}       
    \end{cases}
    $$
    (we assume $\bar{\pi}^+$ is sorted according to the time as per the semantics, and the order of simultaneous elements is arbitrary).

    Let $\delta$ be a rational such that for any $0 \le j \le m$:
    $$
        \frac{t^s_j}{\delta} \in \mathbb{Z}
    $$
    (For example, we can define $\delta$ as the GCD of all numerators of the $t^s_j$’s divided by the LCM of their denominators, since we model timings as rational numbers. However, any choice of $\delta$ that makes the division yield an integer would work.)

    Given a time $t \in \mathbb{Q}_{\ge 0}$, we define the index before , written $ib(t)$, as the largest $j$ s.t. $t^s_j \le t$.

    We prove that $(\bar{\pi}^+, t_e)$ is a valid plan for $\bar{\Pi}$ when using $\delta$ as discretization step.

    The plan is obviously well-formed, because each time is a multiple of $\delta$. 

    We define a sequence of states $\bar{s}_0 \ldots, \bar{s}_{\bar{m}}$ with $\bar{m} \defas \frac{t_e}{\delta}$, where $\bar{s}_0 = \bar{I}$ and for all $0 < j < \bar{m}$ we define $\bar{s}_j(x)$ per cases.
    \begin{itemize}
        \item $s_{ib(j \delta)}(x)$ if $x \in F \cup X$: original fluents values are aligned with the temporal trace. The $j$-th PDDL+ state corresponds to time $\delta j$ in the temporal trace.
        \item $\bar{s}_j(\alive) \defas \top$: \alive is always true.
        \item $\bar{s}_j(\gc) \defas \delta$: due to $\bar{p}^\mutex$, at the beginning of each step \gc is set to $\delta$, then $\bar{e}^{\mutex}$ immediately sets it to 0.
        \item The \running{a} predicate is set to true while an action is running: 
        $$
        \bar{s}_j(\running{a}) \defas
        \begin{cases}
            \top & \mbox{if } \exists i . t_i < j \delta \le t_i+d_i \mbox{ and } a_i = a\\
            \bot & \mbox{otherwise}
        \end{cases}
        $$
        \item Similarly, the $\opencounter$ fluent is set to the number of actions opened and not yet closed \emph{before} time $\delta j$:
        $$
        \bar{s}_j(\opencounter) \defas \sum_{k = 0}^{ib(\delta j) - 1}  |\{\snapstart{a} \in \happenings{k}\}|  - |\{\snapend{a} \in \happenings{k}\}|
        $$
        \item The \clock{a} fluent evolves according to $\bar{p}_a$, fluent, so we simply define:
        $$
        \bar{s}_j(\clock{a}) \defas
        \begin{cases}
            \bar{s}_{j-1}(\clock{a}) + \delta & \mbox{if } \bar{s}_j(\running{a}) \\            
            \bar{s}_{j-1}(\clock{a}) & \mbox{otherwise}
        \end{cases}
        $$
        \item The lock predicates \alock{f}, \ilock{f} and \rlock{f} are set to false when there is an happening at index $k$  that assigns, increases or reads $f$ with $t^s_k = \delta j$:
        $$
        \bar{s}_j(\alock{f}) \defas
        \begin{cases}
            \bot & \mbox{if $\exists k, a$ s.t. }  t^s_k = \delta j, a \in \happenings{k} \mbox{ and } f \in \vrsassign{a}\\
            \top & \mbox{otherwise}
        \end{cases}
        $$
        $$
        \bar{s}_j(\ilock{f}) \defas
        \begin{cases}
            \bot & \mbox{if $\exists k, a$ s.t. }  t^s_k = \delta j, a \in \happenings{k} \mbox{ and } f \in \vrsincrease{a}\\
            \top & \mbox{otherwise}
        \end{cases}
        $$
        $$
        \bar{s}_j(\rlock{f}) \defas
        \begin{cases}
            \bot & \mbox{if $\exists k, a$ s.t. }  t^s_k = \delta j, a \in \happenings{k} \mbox{ and } f \in \vrsread{a}\\
            \top & \mbox{otherwise}
        \end{cases}
        $$
    \end{itemize}

    The sequence of states $\bar{s}_0 \ldots, \bar{s}_{\bar{m}}$ constructed in this way satisfies all the PDDL+ semantic constraints for $\bar{\Pi}$. 
\end{proof}

}{}





\begin{table}[t]
    \scriptsize
    \centering
        \resizebox{\columnwidth}{!}{%
        \renewcommand*{\arraystretch}{0.9}
        \begin{tabular}{@{$\;\;$}l@{$\;\;$}|@{$\;\;$}c|@{$\;\;$}c|@{$\;\;$}c@{$\;\;$}|@{$\;\;$}c@{$\;\;$}|@{$\;\;$}c@{$\;\;$}|@{$\;\;$}c@{$\;\;$}|@{$\;\;$}c@{$\;\;$}|@{$\;\;$}c@{$\;\;$}}
        Domain & ENHSP & ENHSP & ARIES & OPTIC & TAMER & TFLAP & Next- & Patty \\
               & LG    & WA    &       &       &       &       & FLAP          \\
        \hline
        MatchCellar  & 7 & 12 & \textbf{20} & 9 & 7 & \textbf{20} & 2 & 4 \\
        MaJSP  & \textbf{20} & 19 & 18 & N/A & \textbf{20} & N/A & N/A & N/A \\
        T-Plant-Wat & \textbf{20} & 16 & 15 & 20 & 12 & 0 & 0 & 13 \\
        T-Sailing & 11 & \textbf{20} & 6 & 7 & 3 & 2 & 7 & 2 \\
        \hline
        Total (80) & 58 & \textbf{67} & 59 & 36 & 42 & 22 & 9 & 19 \\
        \end{tabular}%
        }
    \caption{Coverage analysis domain by domain, planner by planner. Bold for best, N/A for Not Applicable.}    \label{fig:coverage}
\end{table}

\begin{figure}[tb]
    \centering
    \begin{minipage}{0.55\columnwidth}
        \centering
        \includegraphics[width=\textwidth]{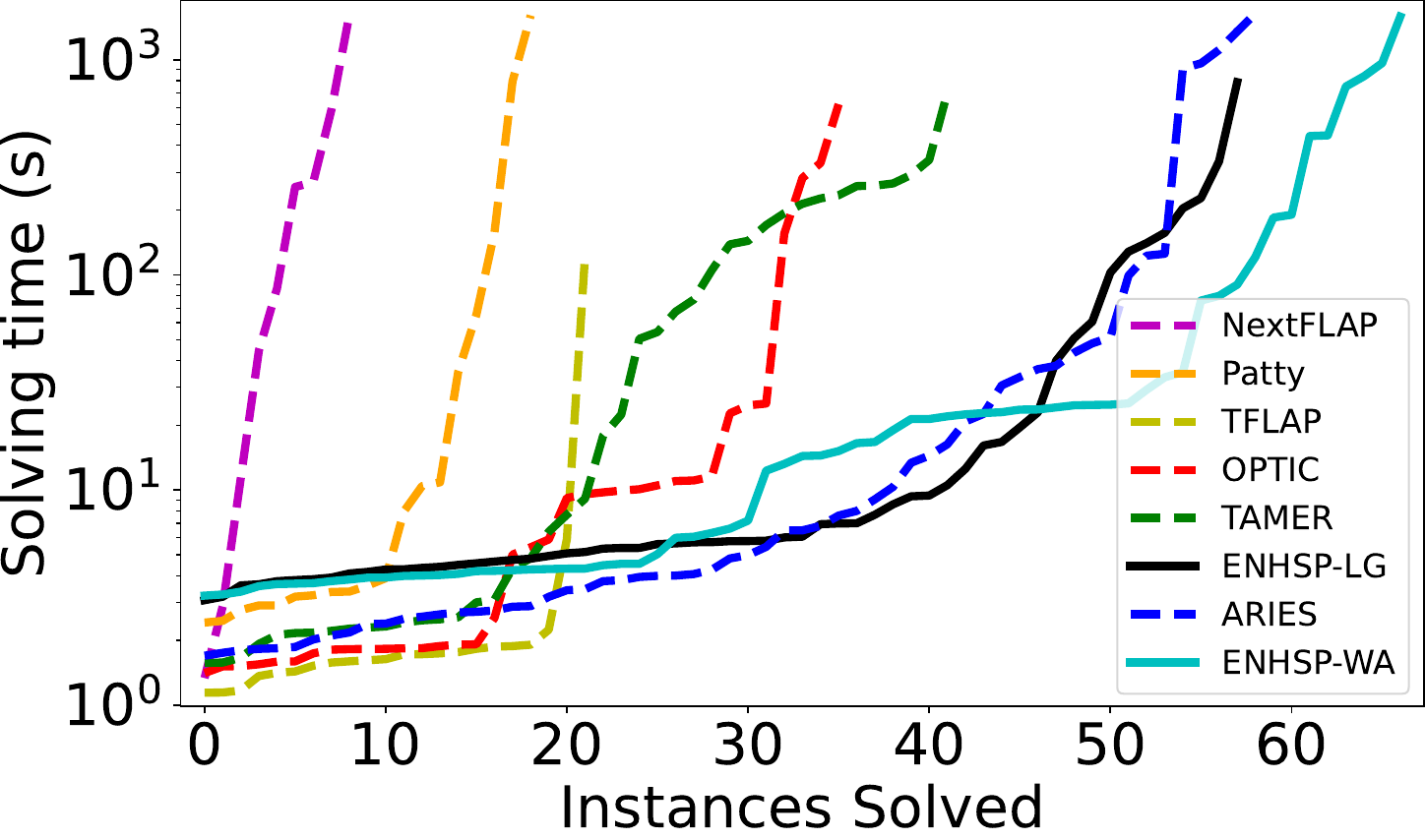}
    \end{minipage}
    \hfill
    \begin{minipage}{0.4\columnwidth}
        \centering
        \includegraphics[width=\textwidth]{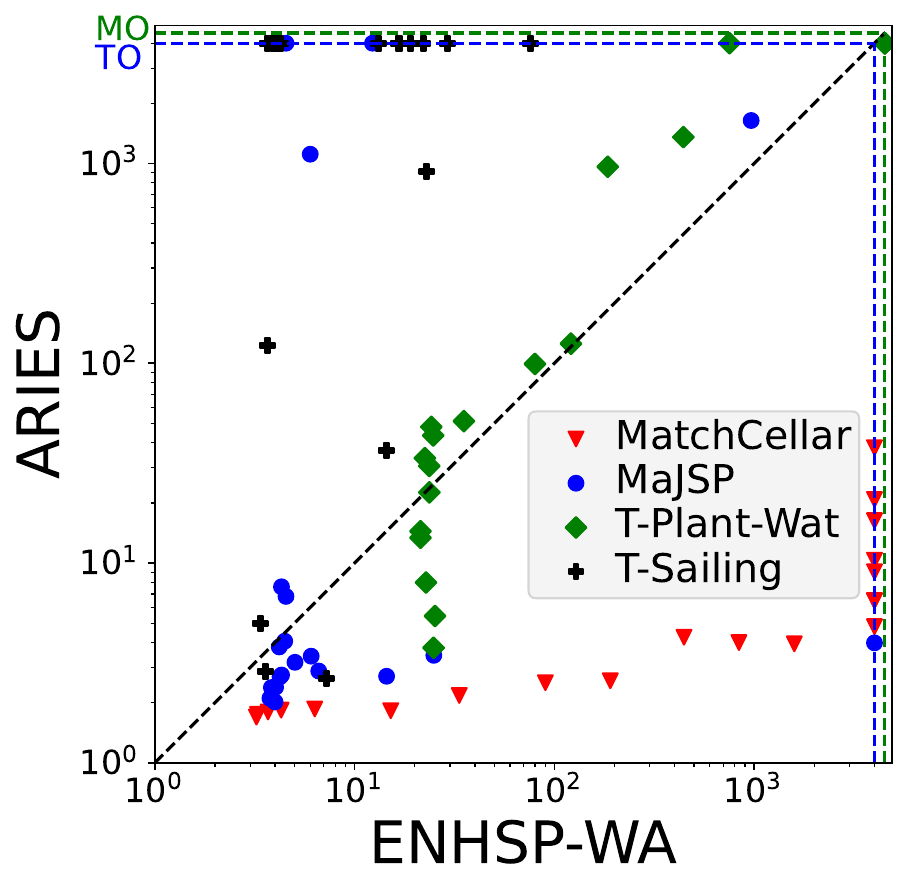}
    \end{minipage}
    \caption{Cactus (survival) plot (left) and run-time scatter plot for the planners with highest coverage (right).}
    \label{fig:combined}
\end{figure}

\section{Experimental Evaluation}

We experimented with our compilation over a selection of temporal numeric domains. We focused our attention on problems requiring intertwined reasoning between numeric and temporal aspects, where the concurrency of the durative actions is necessary to solve the instances. As representative of temporally interesting domains, we took the classic Matchcellar IPC domain and MAJSP from \cite{tpack}. Then we introduce two new domains, T-Sailing, and T-Plant-Watering. T-Sailing extends Sailing \cite{scala:16:ijcai} by requiring a boat not only to rescue the persons in some specific area of the cartesian space, but also to do so under a specific deadline. If the boat arrives too late, the person cannot be saved anymore. T-Plant-Watering extends Plant-Watering \cite{guillem:15:fstrips} by imposing temporal constraints between pouring and opening the tap. The task becomes a collaborative activity involving two distinct agents: one carries the pump used to water the plants, but can begin watering only when the other agent simultaneously performs the task of opening the tap. All benchmarks are available at \url{https://github.com/hstairs/time2processes}.

For each domain we have 20 instances, mostly scaling with the number of objects.
Our analysis focuses on coverage (number of solved instances per domain) and run-time.
ENHSP is used as the PDDL+ planner, run with two different engines, i.e., lazy greedy best-first search (ENHSP-LG) and WA$^*$ (ENHSP-WA), both with the $h^{mrp}$ heuristic \cite{scala:20:hmrp}. In LG, we used focus search as in \citet{scala:25:lgbfs}; in WA$^*$ we use $w=4$. The compiler is implemented within the \texttt{unified\_planning} library \cite{micheli:24:up}, and also supports delayed effects and timed initial literals; roughly, we emulate them with events triggered at the proper time (we omit the formal description due to space constraints).  
We compare the compilation with native state-of-the-art temporal planners, i.e., ARIES \cite{aries}, NextFLAP and TFLAP \cite{nextflap}, OPTIC \cite{optic}, TAMER \cite{tamer} and Patty \cite{patty}.
Experiments were run on an AMD EPYC 7413; 1800s timeout, 20 GB memory limit.

\textbf{Results.} Figure \ref{fig:coverage} shows per-domain coverage. ENHSP-WA got the highest coverage. (Some planners do not support MAJSP for lack of delayed effects support.) For purely temporal domains, temporal planners are faster, yet both ENHSP engines proved competitive, especially in MAJSP. Over temporal numeric domains, ENHSP-WA provided superior performance overall.
Figure \ref{fig:combined} (right) shows a pairwise analysis on run-time for the two best performing planners ARIES and ENHSP-WA. ARIES scales better in Matchcellar, ENHSP-WA better over the temporal numeric domains, highlighting a great deal of complementarity. Finally, Figure \ref{fig:combined} (left) shows the number of instances solved over time,
confirming the strength of our compilation.

\ifthenelse{\boolean{arxiv}}{
\section{Conclusion}

In this paper we presented a compilation from PDDL 2.1 level 3 into PDDL+, proving its soundness and completeness. It was commonly known that PDDL+ could express durative actions, but in thi spaper we provide the first formal account of this fact, providing mechanisms to deal with the subtleties of the PDDL semantics. In particular, we provide a mechanism to faithfully impose durative conditions and the ``no-moving-target rule'' presented in the PDDL 2.1 paper \cite{pddl21}. Our compilation is not only of theoretical interest, but it is also shown to be useful on complex temporal numeric planning problems.

As future work, we would like to extend the compilation to PDDL 2.1. level 4, hence including continuous change. Moreover, we would like to study an alternative compilation targeting PDDL+ with a continuous time semantics.
}{
\newpage
}

\section*{Acknowledgments}
Andrea Micheli and Alessandro Valentini have been partially supported by the STEP-RL project funded by the European Research Council under GA n. 101115870. Enrico Scala has been supported by the Italian Ministry of University and Research within the PRIMA 2024 programme project "Optimizing Water Resources in Coastal Areas using Artificial Intelligence" (AI4WATER -- D53C25000510006)

\bibliography{refs}

@article{gigante22,
  author       = {Nicola Gigante and
                  Andrea Micheli and
                  Angelo Montanari and
                  Enrico Scala},
  title        = {Decidability and complexity of action-based temporal planning over
                  dense time},
  journal      = {Artif. Intell.},
  volume       = {307},
  pages        = {103686},
  year         = {2022},
  url          = {https://doi.org/10.1016/j.artint.2022.103686},
  doi          = {10.1016/J.ARTINT.2022.103686},
  bibsource    = {dblp computer science bibliography, https://dblp.org}
}

@inproceedings{scala:25:lgbfs,
	author = {Scala, Enrico and Bonassi, Luigi},
	title = {On Using Lazy Greedy Best-First Search with Subgoaling Relaxation in Numeric Planning Problems},
	year = {2025},
	booktitle = {Proceedings International Conference on Automated Planning and Scheduling, {ICAPS} 2025},
	pages = {245 – 249},
	doi = {10.1609/icaps.v35i1.36125}
}

@inproceedings{scala:20:hmrp,
  author       = {Enrico Scala and
                  Alessandro Saetti and
                  Ivan Serina and
                  Alfonso Emilio Gerevini},
  title        = {Search-Guidance Mechanisms for Numeric Planning Through Subgoaling
                  Relaxation},
  booktitle    = {Proceedings International Conference on Automated Planning and Scheduling, {ICAPS} 2020},
  pages        = {226--234},
  publisher    = {{AAAI} Press},
  year         = {2020}
}

@article{pddl21,
  title={{PDDL2.1}: An extension to {PDDL} for expressing temporal planning domains},
  author={Maria Fox and Derek Long},
  journal={Journal of artificial intelligence research},
  year={2003}
}

@article{pddlplus,
  author	= {Maria Fox and Derek Long},
  title		= {Modelling Mixed Discrete-Continuous Domains for Planning},
  journal	= {Journal of Artificial Intelligence Research},
  year		= {2006}
}

@inproceedings{pddlplus-discrete,
    title={{On the Notion of Plan Quality for PDDL+}},
    volume={35},
    url={https://ojs.aaai.org/index.php/ICAPS/article/view/36106},
    DOI={10.1609/icaps.v35i1.36106},
    booktitle={Proceedings of the International Conference on Automated Planning and Scheduling},
    author={Percassi, Francesco and Scala, Enrico and Vallati, Mauro},
    year={2025},
    month={Sep.},
    pages={102-111}
}

@inproceedings{popf,
	author    = {Amanda Jane Coles and
	Andrew Coles and
	Maria Fox and
	Derek Long},
	title     = {Forward-Chaining Partial-Order Planning},
	booktitle = {Proceedings International Conference on Automated Planning and Scheduling, {ICAPS} 2010},
	year      = {2010}
}

@inproceedings{optic,
	author    = {J. Benton and
	Amanda Jane Coles and
	Andrew Coles},
	title     = {Temporal Planning with Preferences and Time-Dependent
	Continuous Costs},
	booktitle = {Proceedings International Conference on Automated Planning and Scheduling, {ICAPS} 2012},
	year      = {2012}
}

@inproceedings{tamer,
  title     = {Temporal Planning with Intermediate Conditions and Effects},
  author    = {Alessandro Valentini and
               Andrea Micheli and
               Alessandro Cimatti},
  booktitle = {{AAAI-20} Conference on Artificial Intelligence},
  year      = {2020}
}

@article{micheli:24:up,
  author       = {Andrea Micheli and
                  Arthur Bit{-}Monnot and
                  Gabriele R{\"{o}}ger and
                  Enrico Scala and
                  Alessandro Valentini and
                  Luca Framba and
                  Alberto Rovetta and
                  Alessandro Trapasso and
                  Luigi Bonassi and
                  Alfonso Emilio Gerevini and
                  Luca Iocchi and
                  F{\'{e}}lix Ingrand and
                  Uwe K{\"{o}}ckemann and
                  Fabio Patrizi and
                  Alessandro Saetti and
                  Ivan Serina and
                  Sebastian Stock},
  title        = {Unified Planning: Modeling, manipulating and solving {AI} planning
                  problems in Python},
  journal      = {SoftwareX},
  volume       = {29},
  pages        = {102012},
  year         = {2025},
  url          = {https://doi.org/10.1016/j.softx.2024.102012},
  doi          = {10.1016/J.SOFTX.2024.102012},
  bibsource    = {dblp computer science bibliography, https://dblp.org}
}

@inproceedings{scala:16:ijcai,
  author       = {Enrico Scala and
                  Patrik Haslum and
                  Sylvie Thi{\'{e}}baux},
  editor       = {Subbarao Kambhampati},
  title        = {Heuristics for Numeric Planning via Subgoaling},
  booktitle    = {Proceedings of the Twenty-Fifth International Joint Conference on Artificial Intelligence, {IJCAI} 2016},
  pages        = {3228--3234},
  publisher    = {{IJCAI/AAAI} Press},
  year         = {2016},
  url          = {http://www.ijcai.org/Abstract/16/457}
}

@inproceedings{guillem:15:fstrips,
  author       = {Guillem Franc{\`{e}}s and
                  Hector Geffner},
  editor       = {Ronen I. Brafman and
                  Carmel Domshlak and
                  Patrik Haslum and
                  Shlomo Zilberstein},
  title        = {Modeling and Computation in Planning: Better Heuristics from More Expressive Languages},
  booktitle    = {Proceedings of the Twenty-Fifth International Conference on Automated Planning and Scheduling, {ICAPS} 2015},
  pages        = {70--78},
  publisher    = {{AAAI} Press},
  year         = {2015},
  url          = {http://www.aaai.org/ocs/index.php/ICAPS/ICAPS15/paper/view/10613}
}

@inproceedings{tpack,
  author       = {Andrea Micheli and
                  Enrico Scala},
  title        = {Temporal Planning with Temporal Metric Trajectory Constraints},
  booktitle    = {The Thirty-Third {AAAI} Conference on Artificial Intelligence, {AAAI} 2019},
  pages        = {7675--7682},
  publisher    = {{AAAI} Press},
  year         = {2019},
  url          = {https://doi.org/10.1609/aaai.v33i01.33017675},
  doi          = {10.1609/AAAI.V33I01.33017675}
}

@inproceedings{aries,
  author       = {Arthur Bit{-}Monnot},
  editor       = {Kobi Gal and
                  Ann Now{\'{e}} and
                  Grzegorz J. Nalepa and
                  Roy Fairstein and
                  Roxana Radulescu},
  title        = {Enhancing Hybrid {CP-SAT} Search for Disjunctive Scheduling},
  booktitle    = {{ECAI} 2023 - 26th European Conference on Artificial Intelligence},
  pages        = {255--262},
  publisher    = {{IOS} Press},
  year         = {2023},
  url          = {https://doi.org/10.3233/FAIA230278},
  doi          = {10.3233/FAIA230278}
}

@article{nextflap,
  author       = {{\'{O}}scar Sapena and
                  Eva Onaindia and
                  Eliseo Marzal},
  title        = {A hybrid approach for expressive numeric and temporal planning with
                  control parameters},
  journal      = {Expert Syst. Appl.},
  volume       = {242},
  pages        = {122820},
  year         = {2024},
  url          = {https://doi.org/10.1016/j.eswa.2023.122820},
  doi          = {10.1016/J.ESWA.2023.122820},
  bibsource    = {dblp computer science bibliography, https://dblp.org}
}

@inproceedings{patty,
  author       = {Matteo Cardellini and
                  Enrico Giunchiglia},
  title        = {Temporal Numeric Planning with Patterns},
  booktitle    = {{AAAI-25} Conference on Artificial Intelligence},
  pages        = {26481--26489},
  publisher    = {{AAAI} Press},
  year         = {2025},
  url          = {https://doi.org/10.1609/aaai.v39i25.34848},
  doi          = {10.1609/AAAI.V39I25.34848}
}
\end{document}